\documentclass[conference]{IEEEtran}

\usepackage[letterpaper, left=1.01in, right=1.01in, bottom=1in, top=0.75in]{geometry}

\usepackage[utf8]{inputenc}

\normalsize
\DeclareMathAlphabet\mathbfcal{OMS}{cmsy}{b}{n}

\usepackage[english]{babel}
\usepackage{amsmath}
\usepackage{cite}
\usepackage{graphicx}
\usepackage{dsfont}
\usepackage{amsthm}
\usepackage{amssymb}
\usepackage{mdframed}
\newcommand{\idle}{\mathrm{i}}
\newcommand{\retx}{\mathrm{x}}
\newcommand{\new}{\mathrm{n}}
\usepackage{algorithm}
\usepackage{algorithmic}

\usepackage{lscape}
\usepackage{bbm}
\DeclareMathOperator*{\argmin}{arg\,min} 
\newtheorem{theorem}{Theorem}
\newtheorem{problem}{Problem}

\newcommand{\Exp}[1]{\mathbb{E}\left[ #1 \right]} 
\newcommand{\Exppi}[1]{\mathbb{E}\left[ #1 \right]} 
\title{A Reinforcement Learning  Approach to \linebreak Age of Information in Multi-User Networks}

\author{
\IEEEauthorblockN{Elif Tu\u{g}\c{c}e Ceran, Deniz G{\"u}nd{\"u}z, and Andr\'as Gy\"orgy}
\IEEEauthorblockA{Department of Electrical and Electronic Engineering \\ Imperial College London\\
Email: \{e.ceran14, d.gunduz, a.gyorgy\}@imperial.ac.uk
}}

\begin{document}

\maketitle

\begin{abstract}
Scheduling the transmission of time-sensitive data to multiple users over error-prone communication channels is studied with the goal of minimizing the long-term average \textit{age of information} (AoI) at the users under a constraint on the average number of transmissions at the source node. After each transmission, the source receives an instantaneous ACK/NACK feedback from the intended receiver, and decides on what time and to which user to transmit the next update. The optimal scheduling policy is first studied under different feedback mechanisms when the channel statistics are known; in particular, the standard automatic repeat request (ARQ) and hybrid ARQ (HARQ) protocols are considered. Then a \textit{reinforcement learning} (RL) approach is introduced, which does not assume any a priori information on the random processes governing the channel states.  Different RL methods are verified and compared through numerical simulations.
\end{abstract}

\section{Introduction}

We consider a source node that communicates the most up-to-date status packets to multiple users (see Figure \ref{fig:system}). We are interested in the average \textit{age of information} (AoI) \cite{Altman2010, Kaul2011, Kaul2012} at the users, for a system in which the source node samples an underlying time-varying process and schedules the transmission of the sample values over imperfect links.  The AoI at each user at any point in time can simply be defined as the amount of time elapsed since the most recent status update at that user was generated. Most of the earlier work on AoI consider queue-based models, in which the status updates arrive at the source node randomly following a memoryless Poisson process, and are stored in a buffer before being transmitted to the destination \cite{Kaul2011, Kaul2012}. Instead, in the so-called \emph{generate-at-will} model \cite{Sun2016, Altman2010,Tan2015,Kadota2018,hsuage2017}, also considered in this paper,  the status updates of the underlying process of interest can be generated at any time by the source node. 

AoI in multi-user networks has been studied in \cite{He2017,hsuage2017,Kadota2018, Yates2017,Kaul_multiaccess,Yates_multicast}. It is shown in \cite{He2017} that the scheduling problem for the age minimization is NP-hard in general. Scheduling transmissions to multiple receivers is investigated in \cite{hsuage2017}, focusing on a perfect transmission medium, and the optimal scheduling algorithm is shown to be threshold-type. Average AoI has also been studied when status updates over unreliable multi-access channels \cite{Kaul_multiaccess} and  multi-cast networks \cite{Yates_multicast} are considered. A base station sending time-sensitive information to a number of users through unreliable channels is considered in \cite{Kadota2018}, where the problem is formulated as a multi-armed restless bandit. AoI in the presence of retransmissions has been considered in  \cite{Najm2017, Yates2017}. The status update system is modeled as an M/G/1/1 queue in \cite{Najm2017}, where the status update arrivals are assumed to be memoryless and random. Maximum distance separable (MDS) coding is considered in \cite{Najm2017}, and the successful decoding probabilities are derived in closed form.

\begin{figure}[!t]
\centering
\includegraphics[scale=0.3]{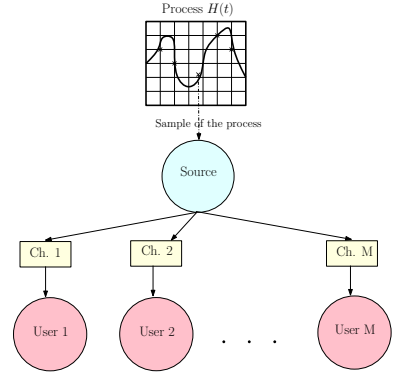}
\caption{The system model of a status update system over error prone links in a multi-user network.}
\label{fig:system}
\end{figure}

In this paper, we address the scheduling of status updates in a multi-user network for both the standard ARQ and HARQ protocols. Our goal is to minimize the expected average AoI under an average transmission-rate constraint. This constraint is motivated by the fact that sensors sending status updates have usually limited energy supplies (e.g., are powered via energy harvesting \cite{Gunduz2014}); hence, they cannot afford to send an unlimited number of updates, or increase the signal-to-noise-ratio in the transmission. First, we assume that the success probability before each transmission attempt is known; hence, the source can judiciously decide when to retransmit, and when to discard failed information and send a fresh update. Then, we consider scheduling status updates over  unknown channels, in which the success probabilities of transmission attempts are not known \textit{a priori}, and must be learned in an online fashion using the ACK/NACK feedback signals. 

In previous work \cite{wcnc_paper}, we have studied a point-to-point status update system in the presence of transmission errors and resource constraint. Here, the results obtained in \cite{wcnc_paper} are extended to the multi-user setting; in addition, more sophisticated reinforcement learning (RL) algorithms are proposed to minimize the average AoI and are demonstrated to perform very close to a lower bound.

The rest of the paper is organized as follows.  In Section~\ref{sec:system}, the system model is presented and the problem of minimizing the average AoI in multi-user networks under a resource constraint is formulated as a \emph{constrained Markov decision process} (CMDP). After determining the structure of the optimal policy, a primal-dual algorithm is proposed  to solve this CMDP in Section~\ref{sec:solution}. Minimization of the AoI for the standard ARQ protocol is investigated in Section~\ref{sec:arq}, and a lower bound on the average AoI is presented. Section~\ref{sec:learning} introduces RL algorithms to minimize the AoI in an unknown environment.  Simulation results are presented in Section~\ref{sec:results}, and the paper is concluded in Section~\ref{sec:conclusion}.

\section{System Model and Problem Formulation}
\label{sec:system}

We consider a slotted status update system where multiple users await time-sensitive information regarding a time-varying process. The source monitors the underlying time-varying process, for which it is able to generate a status update at the beginning of each time slot. The source can only transmit the status update to a single user at each time slot.  This can be either because of dedicated orthogonal links to the users, e.g., a wired network, or because the users are interested in distinct processes. A transmission attempt of a status update to a single user takes constant time, which is assumed to be equal to the duration of one time slot. 

 We assume that the channel state changes randomly from one time slot to the next in an independent and identically distributed fashion. We further assume the availability of an error- and delay-free single-bit ACK/NACK feedback from each user to the source node. 


Let $M$ denote the number of users and $j$ denote the index for each user $j \in \{1,\ldots,M \}$. The AoI for each user is defined as the time elapsed since the most up-to-date packet they received had been generated at the source. Assume that the most up-to-date packet at the destination at time $t$  has a time stamp of generation $U_j(t)$ for the $j^{th}$ user, then the AoI for user $j$ at the beginning of time slot $t$, denoted by $\delta_{j,t}\in \mathds{Z}^+$, is defined as $\delta_{j,t}\triangleq t-U_j(t). $
Therefore, $\delta_{j,t}$ increases by one when the source chooses not to transmit to user $j$ or a transmission fails, while it decreases to one (or, to the number of retransmissions in the case of HARQ) when a status update is successfully decoded.

In the classical ARQ protocol, a packet is retransmitted after each NACK feedback, until it is successfully decoded. However, in the AoI framework there is no point in retransmitting a failed out-of-date status packet if it has the same error probability with a fresh status update. Hence, the source always removes a failed status signal, and transmits a fresh status update. On the other hand, in the HARQ protocol, signals from all previous transmission attempts are combined for decoding; and therefore, the probability of error decreases with every retransmission \cite{harq2003}.

Let  $r_{j,t}\in \{0,\ldots,r_{max}\}$ denote the number of previous transmission attempts of the same packet. Then, the state of the system can be described by the vector $s_t \triangleq (\delta_{1,t}, r_{1,t}, \ldots, \delta_{M,t}, r_{M,t})$. At each time slot, the source node takes one of the several actions, denoted by $a \in \mathcal{A}$, where $\mathcal{A}=\{\idle,\new_1,\retx_1,\ldots, \new_M,\retx_M\}$ denotes the set of possible actions. It can i) remain idle ($a=\idle$); ii) generate and transmit a new status update packet to the $j^{th}$ user ($a=\new_j$); or, iii) retransmit the previously failed packet to the $j^{th}$ user ($a=\retx_j$).   Without loss of generality, each user in the network is assumed to have different priority levels represented by the weights $w_j \in \mathbbm{R}^+$ for user $j$. 


For the $j^{th}$ user, the probability of error after $r$ retransmissions, denoted by $g_j(r)$, depends on $r$,  the particular HARQ scheme used for combining multiple transmission attempts, and the  channel quality between the source and user $j$. An empirical method to estimate $g_j(r)$ is presented in \cite{harq2003}. As in any reasonable HARQ strategy, $g_j(r)$ is non-increasing in $r$, i.e., $g_j(r) \geq g_j(r')$ for all $r \leq r'$. To simplify the analysis and meet with practical constraints, we assume that there is a maximum number of retransmissions $r_{max}$.


Note that if  no resource constraint is imposed on the source, remaining idle is clearly a suboptimal action since it does not contribute to decreasing the AoI. However, continuous transmission is typically not possible in practice due to energy or interference constraints. To model these situations, we impose a constraint on the average number of transmissions, denoted by $\lambda \in (0,1]$.

This leads to the CMDP formulation, defined by the 5-tuple $\big(\mathcal{S}, \mathcal{A},  \mathcal{P}, c, d\big)$ \cite{Altman}: The countable set of states  $s\in \mathcal{S} $ and the finite set of actions $a \in \mathcal{A}$ have already been defined. $ \mathcal{P}$ refers to the transition kernel, where  $ \mathcal{P}_{s,s'}(a) = \Pr(s_{t+1}=s' \mid s_t = s, a_t=a)$ is the probability that action ${\displaystyle a}$  in state ${\displaystyle s}$  at time ${\displaystyle t}$  will lead to state ${\displaystyle s'}$ at time ${\displaystyle t+1}$, which will be explicitly defined in (\ref{eq:transitions}). The instantaneous cost function $c: \mathcal{S} \times \mathcal{A} \rightarrow \mathbbm{R}$, which models the weighted sum of AoI for multiple users, is defined as $c(s,a)=\Delta \triangleq (w_1\delta_{1}+\cdots+ w_M\delta_{M})$  for any $s\in \mathcal{S}$, independently of $a \in \mathcal{A}$. The instantaneous transmission cost related to the constraint, $d:\mathcal{S} \times \mathcal{A} \rightarrow \mathbbm{R}$, is independent of the state and depends only on the action $a$, where $d = 0$ if $a=\idle$, and $d=1$, otherwise. The  transition probabilities of the CMDP are given below where $\mathcal{P}_{s,s'}(a)$ is zero elsewhere.
\begin{align}
\mathcal{P}_{s,s'}(a)=
\begin{cases}
1  &\mathrm{ if } ~a=\idle, \delta'_i=\delta_i+1, \\ &r'_i=r_i, \forall i   \\
1-g_j(0)   &\mathrm{ if } ~a=\new_j, \delta'_j=1, r'_j=0, \\ &\delta'_i=\delta'_i+1, r'_i=r_i, \forall i \neq j \\
g_j(0)   &\mathrm{ if } ~a=\new_j, \delta'_j=\delta_j\!+\!1, r'_j\!=\!1 \\ &\delta'_i=\delta'_i+1, r'_i=r_i, \forall i \neq j \\
1-g_j(r_j)   &\mathrm{ if } ~a=\retx_j, \delta'_j=r_j, r'_j=0, \\ &\delta'_i=\delta'_i+1, r'_i=r_i, \forall i \neq j \\
g_j(r_j)   &\mathrm{ if } ~a=\retx_j, \delta'_j=\delta_j+1,\\ &r'_j=r'_j+1, \delta'_i=\delta'_i+1,\\ &r'_i=r_i, \forall i \neq j
\end{cases}
\label{eq:transitions}
\end{align}

A stationary \emph{policy} is a decision rule represented by $\pi: \mathcal{S} \times \mathcal{A} \rightarrow [0,1]$, which maps the state $s\in \mathcal{S}$ into action $a\in \mathcal{A}$ with some probability $\pi(a|s)$ and $\sum_a\pi(a|s)=1$. We will use $s_t^{\pi}=(\delta_{1,t}^{\pi},r_{1,t}^{\pi},\ldots,\delta_{M,t}^{\pi},r_{M,t}^{\pi})$ and $a_t^{\pi}$ to denote the sequences of states and actions, respectively, induced by policy $\pi$ with initial state $s_0$. Let $J^{\pi}(s_0)$ denote the infinite horizon average age, and $C^{\pi}(s_0)$ denote the expected average number of transmissions, when $\pi$ is employed with initial state $s_0$. 
We can state the CMDP optimization problem as follows:
\begin{problem}
\begin{subequations}
\begin{align}
&\underset{\pi\in \Pi}{\mathrm{ Minimize }}  ~J^{\pi}(s_0) \triangleq \limsup_{T\rightarrow \infty }\frac{1}{T}\Exppi{\sum_{t=1}^T{\Delta^{\pi}_t}\Big|s_0}, \label{eq:cost}\\
&\mathrm{ s.t. }  ~C^{\pi}(s_0) \triangleq \limsup_{T\rightarrow \infty }\frac{1}{T}\Exppi{\sum_{t=1}^T{\mathbbm{1}[a^{\pi}_t \neq \idle]  }\Big|s_0}\leq \lambda \label{eq:constraint},
\end{align} 
\end{subequations}
\label{problem}
\end{problem} where  $\Delta^{\pi}_t\triangleq\sum_{j=1}^M {w_j\delta_{j,t}^{\pi}}$. A policy $\pi^*\in \Pi$ is called optimal if $J^*\triangleq J^{\pi^*} \leq J^{\pi}$ for all $\pi\in \Pi$. For a deterministic policy, we will use $\pi(s)$ to denote the action taken with probability one in state $s$. Also, without loss of generality, we assume that the initial state at the beginning of the problem is $s_0=(1,0,2,0,\ldots,M-1,0,M,0)$; and $s_0$ will be omitted from the notation for simplicity.  We also assume throughout this paper that the Markov decision process (MDP) is \emph{unichain} \cite{Altman}, similarly to \cite{wcnc_paper}.

\section{Primal-Dual Algorithm to Minimize AoI}
\label{sec:solution}

In this section, we derive the solution for Problem~\ref{problem}, based on \cite{Altman}.  
While there exits a stationary and deterministic optimal policy for countable-state finite-action average-cost MDPs \cite{Puterman_book}, this is not necessarily true for CMDPs \cite{Altman}. 

To solve the constrained MDP, we start by rewriting Problem \ref{problem} in its Lagrangian form. The average Lagrangian cost of a policy $\pi$ with Lagrange multiplier $\eta \ge 0$, denoted by $J^{\pi}_{\eta}$, is defined as
\begin{align}
\lim_{T\rightarrow \infty }\frac{1}{T}\Exp{\sum_{t=1}^T{\Delta^{\pi}_t}}\!-\!\eta (C_{max}\!-\!\frac{1}{T}\Exp{\sum_{t=1}^T{\mathbbm{1}[a^{\pi}_t\neq \idle]}})
\end{align}
and, for any $\eta$, the optimal achievable cost $J^*_{\eta}$ is defined as $J_{\eta}^*\triangleq \min_{\pi}{J^{\pi}_{\eta}}$.
This formulation is equivalent to an unconstrained average-cost MDP, in which the instantaneous overall cost becomes  $\Delta_t+\eta\mathbbm{1}[a^{\pi}_t \neq \idle]$. It is well-known that there exits an optimal stationary deterministic policy for this problem.  In particular, there exists a function $h_{\eta}(s)$, called the differential cost function, satisfying the so-called \emph{Bellman optimality} equations:
\begin{equation}
\label{eq:Bellman}
h_{\eta}(s)+J^*_\eta=\min_{a\in\mathcal{A}}\big(\Delta+\eta \cdot \mathbbm{1}[a \neq \idle]+\Exp{h_{\eta}(s')}\big), 
\end{equation}
where $s'$ is the next state obtained from $s$ after taking action $a$. 
Then the optimal policy, for any $s \in \mathcal{S}$, is given by the action achieving the minimum in \eqref{eq:Bellman}:
\begin{equation}
\label{eq:opt_eta}
\pi_{\eta}^*(s) \in \argmin_{a\in\mathcal{A}} \big(\Delta+\eta \cdot \mathbbm{1}[a \neq \idle]+\Exp{h_{\eta}(s')}\big). 
\end{equation}
The relative value iteration (RVI) algorithm can be employed to solve \eqref{eq:Bellman} for any given $\eta$; and hence, to find the policy $\pi^*_\eta$ (more precisely, an arbitrarily close approximation) \cite{Puterman_book}. 

Similarly to Corollary 1 in \cite{wcnc_paper}, it is possible to characterize optimal policies for our CMDP problem using the deterministic policies $\pi_\eta^*$,: Specializing Theorem 4.4 of \cite{Altman}  to Problem~\ref{problem} (since it has a single global constraint), one can think of the optimal policy as a randomized policy between two deterministic policies: in any state $s=(\delta,r)$, the optimal policy in the CMDP problem chooses action $\pi^*_{\eta_1}(s)$ with probability $\mu$ and $\pi^*_{\eta_2}(s)$ with probability $1-\mu$ independently for each time slot where $\pi^*_{\eta_i}$ is the probability vector describing the deterministic choice of the optimal policy in the unconstrained MDP with Lagrange multiplier $\eta_i$. 





For any $\eta$, let $C_{\eta}$ denote the average resource consumption under the optimal policy $\pi_{\eta}^*$  (note that $C_\eta$ and $J^*_\eta$ can be computed directly through finding the stationary distribution of the chain, but can also be estimated empirically just by running the MDP with policy $\pi^*_\eta$). Obviously, $C_\eta$ and $J^*_\eta$ are monotone functions of $\eta$. Therefore,  given  $\eta_1$ and $\eta_2$, one can find a weight, denoted by $\mu$, by solving  $\mu C_{\eta_1} + (1-\mu) C_{\eta_2}=\lambda$, which has a solution  $\mu \in [0,1]$ if $C_{\eta_1} \ge \lambda \ge C_{\eta_2}$.

Next, we present a heuristic method to find $\eta_1$ and $\eta_2$: With the aim of finding a single $\eta$ value such that $C_\eta \approx \lambda$, starting with an initial parameter $\eta^0$, we run an iterative algorithm updating $\eta$ as $\eta^{m+1} = \eta^m+\alpha (C_{\eta^m}-\lambda)$ for some step size parameter $\alpha \triangleq 1/\sqrt{m}$. We continue this iteration until $|\eta^{m+1}-\eta^m|$ is smaller than a given  $\epsilon\in \mathbbm{R}^+$,  and denote the resulting value as $\eta^*$. Then, we approximate the values of $\eta_1$ and $\eta_2$ by $\eta^*\pm \xi$, where
$\xi$ is a small perturbation and the mixture policy can obtained as:
\begin{equation}
\label{eq:piopt}
\pi^*_{\lambda}=\mu \pi^*_{\eta_1}+ (1-\mu) \pi^*_{\eta_2}.
\end{equation}

\section{AoI with Classical ARQ Protocol}
\label{sec:arq}

Now, assume that the system adopts the classical ARQ protocol; that is, failed transmissions are discarded at the destination. In this case, there is no point in retransmitting a failed packet since the successful transmission probabilities are the same for a retransmission and the transmission of a new update. The state space reduces to  $(\delta_1, \delta_2, \ldots, \delta_M)$ as $r_{j,t}=0,~\forall j,t$, and the action space to $\mathcal{A}\in\{\idle , \new_1, \ldots, \new_M\}$. The probability of error of each status update is $p_j\triangleq g_j(0)$ for user $j$. State transitions in \eqref{eq:transitions}, Bellman optimality equations and the RVI algorithm can all be simplified accordingly.  Thanks to these simplifications, we are able to provide a closed-form  lower bound to the constrained MDP.

\subsection{Lower Bound on the AoI under Resource Constraint}

In this section, we derive a lower bound to the average AoI for the multi-user network with standard ARQ protocol.
\begin{theorem}
For a given network setup, we have $J_{LB}\leq J^{\pi}$, $\forall \pi \in \Pi$, where
\begin{align}
&J_{LB}=\frac{1}{2\lambda}{\left(\sum_{j=1}^M {\sqrt{\frac{w_j}{1-p_j}}}\right)}^2+\frac{\lambda w_{j^*}p_{j^*}}{2(1-p_{j^*})}+\frac{1}{2}\sum_{j=1}^M w_j, \\
&\textnormal{and }  j^*\triangleq \argmin_j{\frac{w_jp_j}{2(1-p_j)}}. \nonumber
\end{align}
\end{theorem}
 
\begin{proof} 

The proof will be provided in the extended version of the paper.
\end{proof}
Previously, \cite{Kadota2018} proposed a universal lower bound on the average AoI for the broadcast channel with multiple users for the special case of $\lambda=1$. Differently from \cite{Kadota2018}, the lower bound derived in this paper shows the effect of constraint ($\lambda$) and even for $\lambda=1$, it is tighter than the lower bound provided in \cite{Kadota2018}.

\section{Learning to minimize AoI in an unknown environment}
\label{sec:learning}

In most practical scenarios, channel error probabilities for retransmissions may not be known at the time of deployment, or may change over time, where the source node does not have \textit{a priori} information about the decoding error probabilities and has to learn them over time. We employ online learning algorithms to learn the error probabilities over time without degrading the performance significantly. 


The Upper Confidence RL (UCRL2) \cite{UCRL2} is a well-known RL algorithm for generic MDP problems which has strong theoretical guarantees with regard to high probability regret bounds. However, the computational complexity of the algorithm scales quadratically with the size of the state space, which makes the algorithm unsuitable for large state spaces. UCRL2 has been initially proposed for generic MDPs with unknown rewards and transition probabilities: thus, they need to be learned for each state-action pair. On the other hand, for the average AoI problem, the number of parameters to be learned can be reduced to the number of transmission error probabilities to each user; thus, the computational complexity can be reduced significantly. In addition, the constrained structure of the average AoI problem requires additional modifications to the UCRL2 algorithm, which is achieved in this paper by updating the Lagrange multiplier according to the empirical resource consumption.  


\subsection{UCRL2 with standard ARQ}

In this section, we consider a multi-user network with standard ARQ where a source node transmits to multiple users with unknown and distinct error probabilities $p(j)\triangleq p_j$. UCRL2  exploits the optimistic MDP characterized by the optimistic estimation of error probabilities  within a certain confidence interval.  The details of the algorithm are given in Algorithm \ref{alg:ARQ}, where  $\widehat{p}(j)$ and $\tilde{p}(j)$ represent the empirical and the optimistic estimate of the error probability for user $j$. 

We propose several methods to find the optimal policy $\tilde{\pi}_k$ using the optimistic estimate $\tilde{p}(j)$ defined in steps 4 and 5 of  Algorithm  \ref{alg:ARQ}. In the generic UCRL2, extended value iteration (VI) is used for steps 4 and 5, which has high computational complexity for large networks. For the average AoI problem, the computational complexity can be reduced since the optimistic MDP can be found easily using the lower bound for the error probabilities and value iteration can be adopted to compute $\pi$ induced by $\tilde{p}(j)$ in step 5. The resulting algorithm will be called as \textit{UCRL2-VI}.

In order to further reduce the computational complexity, we can also adopt a suboptimal Whittle index  policy,  proposed in \cite{Kadota2018}, in  step 5 of the algorithm. The resulting algorithm is called as \textit{UCRL2-Whittle} in this paper and the policy $\pi_k$ in step 5 can be found as follows:  
\begin{itemize}
\item Compute the index for each user (similarly to \cite{Kadota2018}),
\begin{align}
I_j\triangleq w_j(1-\tilde{p}(j))\delta_j\left(\delta_j+\frac{1+\tilde{p}(j)}{1-\tilde{p}(j)}\right).
\end{align}
\item Compare the highest index with the Lagrange parameter $\eta$: if $\eta$ is smaller  then the source transmits to the user with the highest index, otherwise the source idles. 
\end{itemize}

\begin{algorithm}
\begin{footnotesize}
\caption{UCRL2 for the average AoI with standard ARQ.}
\begin{algorithmic}[1]
 \renewcommand{\algorithmicrequire}{\textbf{Input:}}
 \renewcommand{\algorithmicensure}{\textbf{Output:}}
 \REQUIRE A confidence parameter $\delta\in (0,1)$, an update parameter $\alpha$, $\lambda$, confidence bound $U$, $\mathcal{S}$, $\mathcal{A}$.
  \STATE $\eta=0$, $t=1$ and observe the initial state $s_1$.
  \FOR {episodes $k= 1,2,\ldots$ }
  \STATE Set $t_k\triangleq t$,\\ $N_k(j)\triangleq\#\{\tau<t_k:a_{\tau}=\new_j\}$,\\
  $E_k(j)\triangleq\#\{\tau<t_k:a_{\tau}=\new_j, failure\}$\\
  $\widehat{p}(j)\triangleq \frac{E_k(a)}{\max\{N_k(a),1\}}$,\\
  $C_k\triangleq\#\{\tau<t_k:a_{\tau}\neq \idle\}$,\\
  $\eta \leftarrow \eta+\alpha (C_k/t_k-\lambda)$. 
  \STATE Compute the optimistic error probabilities \\
  $\tilde{p}(j)\triangleq\max\{0,\widehat{p}(j)-\sqrt{\frac{U\log(S A t_k/\delta)}{max\{1,N_k(j)\}}}\}$
  \STATE Use $\tilde{p}(j)$  to find a policy $\tilde{\pi}_k$
  \STATE Execute policy $\tilde{\pi}_k$
  \WHILE{$v_k(j)<N_k(j)$}
 \STATE Choose an action $a_t=\tilde{\pi}_k(s_t)$,\\
 Obtain cost $\sum_{j=1}^M w_j\delta_j +\eta* \mathbbm{1}[a_t\neq \idle]$ and observe $s_{t+1}$\\
 Update $v_k(j)=v_k(j)+1$, \\
 Set $t=t+1$;
  \ENDWHILE
  \ENDFOR
 \end{algorithmic} 
  \label{alg:ARQ}
 \end{footnotesize}
\end{algorithm}

\subsection{UCRL2 with HARQ}

The pseudocode of the algorithm is given in Algorithm \ref{alg:UCRL_HARQ}, where $\widehat{g_j}(r)$ and $\tilde{g_j}(r)$ represent the empirical and the optimistic estimates of the error probability for user $j$, after $r$ retransmissions. 

\begin{algorithm}
\begin{footnotesize}
\caption{UCRL2 for the average AoI with HARQ.}
\begin{algorithmic}[1]
 \renewcommand{\algorithmicrequire}{\textbf{Input:}}
 \renewcommand{\algorithmicensure}{\textbf{Output:}}
 \REQUIRE A confidence parameter $\delta\in (0,1)$, an update parameter $\alpha$, $\lambda$, $\mathcal{S}$, $\mathcal{A}$.
  \STATE $\eta=0$, $t=1$ and observe the initial state $s_1$.
  \FOR {episodes $k= 1,2,\ldots$ }
  \STATE Set $t_k\triangleq t$,\\ $N_k(j,r)\triangleq\#\{t<t_k:a_t=\retx_j,r_{j,t}=r\}$,\\
  $N_k(j,0)\triangleq\#\{t<t_k:a_t=\new_j\}$,\\
  $E_k(j,r)\triangleq\#\{\tau<t_k:a_{\tau}=\retx_j, r_{j,t}=r, failure\}$\\
  $E_k(j,0)\triangleq\#\{\tau<t_k:a_{\tau}=\new_i,  failure\}$\\
  $\widehat{g_j}(r)\triangleq \frac{E_k(j,r)}{\max\{N_k(j,r),1\}}$,\\
  $C_k\triangleq\#\{\tau<t_k:a_{\tau}\neq \idle\}$,\\
  $\eta \leftarrow \eta+\alpha (C_k/t_k-\lambda)$. 
  \STATE Compute the optimistic error probabilities \\
  $\tilde{g_j}(r)\triangleq\max\{0,\widehat{g_j}(r)-\sqrt{\frac{U\log(S A t_k/\delta)}{max\{1,N_k(j,r)\}}}\}$
  \STATE Use $\tilde{g_j}(r)$ and value iteration to find a policy $\tilde{\pi}_k$
  \STATE Execute policy $\tilde{\pi}_k(s_t)$
  \WHILE{$v_k(j,r)<N_k(j,r)$}
 \STATE Choose an action $a_t=\tilde{\pi}_k(s_t)$,\\
 Obtain cost $\sum_{j=1}^M w_j\delta_j +\eta* \mathbbm{1}[a_t\neq \idle]$ and observe $s_{t+1}$\\
 Update $v_k(j,r)=v_k(j,r)+1$, \\
 Set $t=t+1$;
  \ENDWHILE
  \ENDFOR
 \end{algorithmic} 
  \label{alg:UCRL_HARQ}
 \end{footnotesize}
\end{algorithm}

\section{Numerical Results}
\label{sec:results}

First, we analyze the average AoI in a multi-user setting with standard ARQ protocols. The average AoI for a given resource constraint $\lambda$ is illustrated in Figure~\ref{fig:cons_vs_AoI} for a 3-user network with error probabilities given as $p=[0.5~ 0.2~ 0.1]$.  It can be seen from Figure~\ref{fig:cons_vs_AoI} that both UCRL2-VI and UCRL2-Whittle perform very close to lower bound particularly when $\lambda$ is low, i.e. the system is more constrained. Although UCRL2-Whittle algorithm has a significantly lower computational complexity, it performs  very similar to UCRL-Whittle for all $\lambda$ values. 

\begin{figure}[!t]
\centering
\includegraphics[scale=0.43]{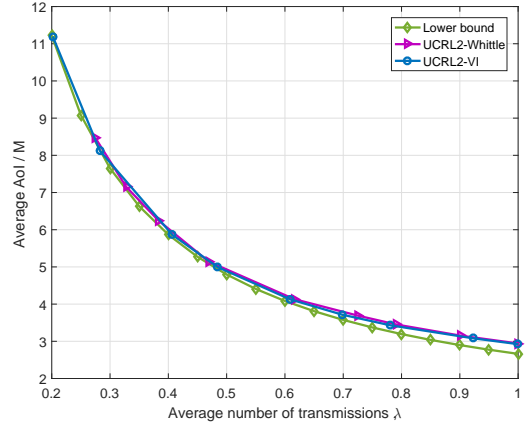}
\caption{Average AoI with respect to $\lambda$ for a 3-user network with $M=3$ and error probabilities $p=[0.5~ 0.2~ 0.1]$, $w_j=1,~\forall j$. Time horizon is set to $T=10^5$, and the results are averaged over $100$ runs.}
\label{fig:cons_vs_AoI}
\end{figure}

Figure~\ref{fig:arq_size} illustrates the average AoI with standard ARQ with respect to the size of a network when there is no constraint on the average number of transmissions (i.e. $\lambda=1$) and the performance of the UCRL2 algorithm is compared with the lower-bound since the computational cost of value/policy iteration algorithms is very high. Learning algorithm performs close lower-bound and very close to the Whittle index policy \cite{Kadota2018} which assumes the a priori knowledge of error probabilities.  Moreover, the UCRL2 algorithm outperforms the greedy benchmark policy which always transmits to the user with the highest age and Round Robin policy which transmits to each user in turns. 

\begin{figure}[!t]
\centering
\includegraphics[scale=0.44]{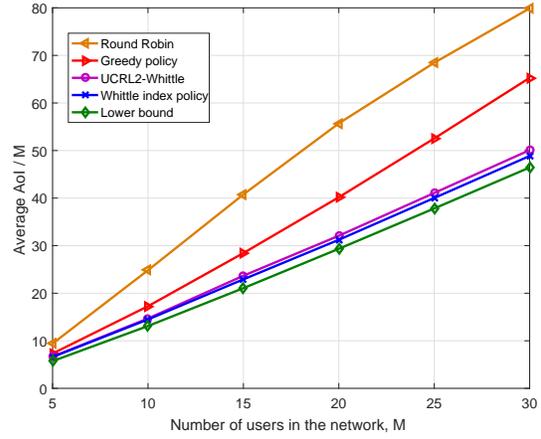}
\caption{Average AoI for networks with different sizes where $p_j=j/M$, $\lambda=1$ and  $w_j=1,~\forall j$.The simulation results are averaged over 100 runs.}
\label{fig:arq_size}
\end{figure}

The performance of UCRL2-Whittle and average cost SARSA are shown in Figure~\ref{fig:learn_ARQ}. UCRL2-Whittle converges much faster compared to the standard Average-cost SARSA algorithm, and it performs very close to the optimal algorithm computed by value iteration (VI) with known error probabilities. Figure~\ref{fig:learn_HARQ} shows the performance of learning algorithms for HARQ protocol for a 2-user scenario. It is worth noting that although UCRL2-VI converges to the optimal policy in fewer iterations than average-cost SARSA, iterations in  UCRL2-VI is computationally more demanding since it uses value iteration in each $k$. Therefore, UCRL2-VI  is not practical for problems with large state spaces, in our case for large networks.


\begin{figure}[!t]
\centering
\includegraphics[scale=0.45]{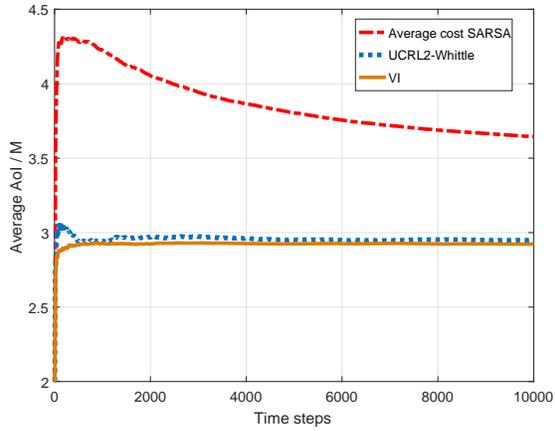}
\caption{Average AoI for networks for a 3-user ARQ network with $M=3$ and error probabilities $p=[0.5~ 0.2~ 0.1]$ where $\lambda=1$ and $w_j=1,~\forall j$. The simulation results are averaged over 100 runs.}
\label{fig:learn_ARQ}
\end{figure}

\begin{figure}[!t]
\centering
\includegraphics[scale=0.45]{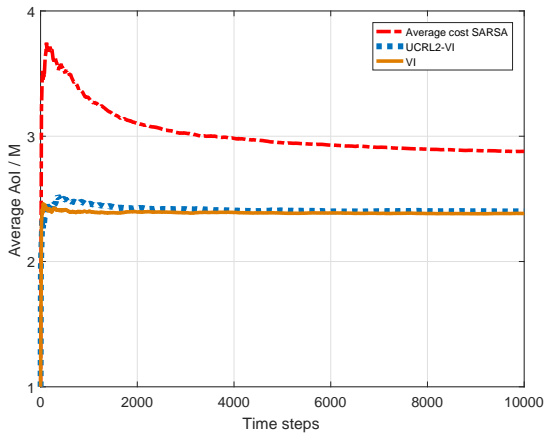}
\caption{Average AoI for networks for a 2-user HARQ network with $M=2$ and error probabilities $g_1(r_1)=0.5 \cdot 2^{r_1}$ and $g_2(r_2)=0.5 \cdot 2^{r_2}$ where $\lambda=1$ and $w_j=1,~\forall j$. The simulation results are averaged over 100 runs.}
\label{fig:learn_HARQ}
\end{figure}

\section{Conclusion}
\label{sec:conclusion}

Scheduling the transmission of status updates to multiple destination nodes has been considered with the average AoI as the performance measure. Under a resource constraint, the problem is modeled as a CMDP considering both the classical ARQ and the HARQ protocols and an online scheduling policy has been proposed. A lower bound on the average AoI has been shown for the standard ARQ protocol. RL algorithms are presented for scenarios when the error probabilities may not be known in advance, and demonstrated to perform very close optimal for scenarios investigated in numerical simulations. The algorithms adopted in this paper are also relevant to different multi-user systems concerning the timeliness of information, and the proposed methodology can be used in other CMDP problems.

\bibliography{ageofmultieski}
\end{document}